\documentclass[graybox]{svmult}
\usepackage[backend=biber,
            url=false,
            isbn=false,
            doi=false,
            citestyle=numeric-comp,
            style=ieee,
            maxbibnames=50,
            maxcitenames=2, 
            mincitenames=1, 
            sorting=nyt,
            block=none]{biblatex}
            
\usepackage{helvet}
\usepackage{courier}
\usepackage{graphicx}
\usepackage{makeidx}
\usepackage{multicol}
\usepackage[bottom]{footmisc}

\usepackage[font={small}]{caption}
\usepackage{subcaption}
\usepackage[rightcaption]{sidecap}
\usepackage{pbox}

\usepackage{mathtools}
\usepackage{amsmath, amssymb, amscd}

\usepackage{amsthm}
\usepackage{ wasysym } 
\usepackage{amsfonts}       
\usepackage{mathptmx} 
\usepackage{gensymb} 
\usepackage{nicefrac}       

\DeclareMathAlphabet{\mathcal}{OMS}{lmsy}{m}{n}
\DeclareSymbolFont{largesymbols}{OMX}{cmex}{m}{n}
\usepackage{textcomp} 

\usepackage{algorithm} 

\usepackage{algorithmic} 
\floatname{algorithm}{Algorithm}
\renewcommand{\algorithmiccomment}[1]{// #1}
\renewcommand{\algorithmicrequire}{\textbf{Input:}}
\renewcommand{\algorithmicensure}{\textbf{Output:}}

\usepackage{booktabs}       
\usepackage{array} 
\usepackage{tabularx}
\usepackage{multirow}
\usepackage{multicol}
\usepackage{booktabs}

\usepackage[utf8]{inputenc}
\usepackage[T1]{fontenc} 
\usepackage[english]{babel} 
\usepackage{units}
\usepackage{bm}
\usepackage{times} 
\usepackage{xspace}
\usepackage{flushend}
\usepackage{csquotes}
\usepackage{makeidx}
\usepackage{enumitem}
\setlist[enumerate]{itemsep=0mm}
\usepackage{soul} 
\usepackage{subfiles} 

\usepackage[yyyymmdd]{datetime}

\date{\protect\formatdate{1}{1}{2001}}

\usepackage{url}
\makeatletter
\g@addto@macro{\UrlBreaks}{\UrlOrds}
\makeatother
\usepackage{color}
\usepackage[usenames,dvipsnames,table]{xcolor}

\usepackage{soul} 

\newcommand{\ignore}[1]{}

\newcommand{\seclabel}[1]{\label{sec:#1}}

\usepackage{type1cm}        

\newcommand*\Eval[3]{\left.#1\right\rvert_{#2}^{#3}}

\newcommand{\algoName}{\textsc{AdaPT}\xspace}
\newcommand{\algoFull}{Adaptive Policy Transfer for Stochastic Dynamics\xspace}

\makeindex             

\begin{document}


\title*{\textsc{AdaPT}: Zero-Shot Adaptive Policy Transfer \\for Stochastic Dynamical Systems}
\titlerunning{\textsc{AdaPT}:Adaptive Policy Transfer} 

\author{James Harrison$^{1}$, Animesh Garg$^{2}$, Boris Ivanovic$^2$, Yuke Zhu$^2$, Silvio Savarese$^2$, Li Fei-Fei$^2$, Marco Pavone$^3$ 
}
\authorrunning{Harrison et al.}


\institute{
James Harrison \at Department of Mechanical Engineering, Stanford University, Stanford, CA 94305\\ \email{jharrison@stanford.edu}
\and Animesh Garg, Boris Ivanovic, Yuke Zhu, Li Fei-Fei, Silvio Savarese \at Department of Computer Science, Stanford University, Stanford, CA 94305\\ \emailbrack{garg,borisi,yukez,feifeili,ssilvio}\url{@cs.stanford.edu}
\and Marco Pavone \at Department of Aeronautics and Astronautics, Stanford University, Stanford, CA 94305\\ \email{pavone@stanford.edu}
}

\maketitle

\setlength{\belowdisplayskip}{2pt} \setlength{\belowdisplayshortskip}{2pt}
\setlength{\abovedisplayskip}{2pt} \setlength{\abovedisplayshortskip}{2pt}


\abstract{Model-free policy learning has enabled good performance on complex tasks that were previously intractable with traditional control techniques. However, this comes at the cost of requiring a perfectly accurate model for training. This is infeasible due to the very high sample complexity of model-free methods preventing training on the target system. This renders such methods unsuitable for physical systems. Model mismatch due to dynamics parameter differences and unmodeled dynamics error may cause suboptimal or unsafe behavior upon direct transfer. We introduce the \algoFull(\algoName) algorithm that achieves provably safe and robust, dynamically-feasible zero-shot transfer of RL-policies to new domains with dynamics error. \algoName combines the strengths of offline policy learning in a black-box source simulator with online tube-based MPC to attenuate bounded dynamics mismatch between the source and target dynamics. \algoName allows online transfer of policies, trained solely in a simulation offline, to a family of unknown targets without fine-tuning. We also formally show that (i) \algoName guarantees bounded state and control deviation through state-action tubes under relatively weak technical assumptions and, (ii) \algoName results in a bounded loss of reward accumulation relative to a policy trained and evaluated in the source environment.
We evaluate \algoName on $2$ continuous, non-holonomic simulated dynamical systems with 4 different disturbance models, and find that \algoName performs between $50\%$-$300\%$ better on mean reward accrual than direct policy transfer.}


\section{Introduction}

Deep reinforcement learning (RL) has achieved remarkable advances in sequential decision making in recent years, often outperforming humans on tasks such as Atari games \cite{mnih2015human}. 
However, model-free variants of deep RL are not directly applicable to physical systems because they exhibit poor sample complexity, often requiring millions of training examples on an accurate model of the environment. 
One approach to using model-free RL methods on robotic systems is thus to train in a relatively accurate simulator (a source domain), and transfer the policy to the physical robot (a target domain). 
This naive transfer may, in practice, perform arbitrarily badly and so online fine-tuning may be performed~\cite{abbeel2006using}. 
During this fine-tuning, the robot may behave unsafely however, and so it is desirable for a system to be able to train in a simulator with slight model inaccuracies but still be able to perform well on the target system on the first iteration. 
We refer to this as the \textit{zero-shot policy transfer problem}.

The zero-shot transfer problem involves training a policy on a system possessing different dynamics than the target system, and evaluating performance as the average initial return in target domain \textit{without} training in the target domain. 
This problem is challenging for robotic systems since simplified simulated models may not always accurately capture all relevant dynamics phenomena, such as friction, structural compliance, turbulence and so on, as well as parametric uncertainty in the model. 
In spite of the renewed focus on this problem, few studies in deep policy adaptation offer insightful analysis or guarantees regarding feasibility, safety, and robustness in policy transfer.

In this paper, we introduce a new algorithm which we refer to as \algoName, that achieves provably safe and robust, dynamically-feasible zero-shot direct transfer of RL policies to new domains with dynamics mismatch. The key insight here is to leverage the global optimality of learned policy with local stabilization from MPC based methods to enable dynamic feasibility, thereby building on strengths of two different methods.
In the offline stage, \algoName first computes a nominal trajectory (without disturbance) by executing the learned policy on the simulator dynamics. Then in the online stage, \algoName adapts the nominal trajectory to the target dynamics with an auxiliary MPC controller. 

\runinhead{Statement of Contributions} 
\begin{enumerate}[
    topsep=0pt,
    noitemsep,
    leftmargin=*,
    itemindent=3ex]
\item We develop the \algoName algorithm, which allows online transfer of policy trained solely in a simulation offline, to a family of unknown targets without fine-tuning.
\item We also formally show that (i) \algoName guarantees state and control safety through state-action tubes under the assumption of Lipschitz continuity of the divergence in dynamics and, (ii) \algoName results in a bounded loss of reward accumulation in case of direct transfer with \algoName as compared to a policy trained only on target. 
\item We evaluate \algoName on two continuous, non-holonomic simulated dynamical systems with four different disturbance models, and find that \algoName performs between $50\%$-$300\%$ better on mean reward accrual than direct policy transfer as compared to mean reward.
\end{enumerate}

\runinhead{Organization} This paper is structured as follows. In Section 2 we review related work in robust control, robust reinforcement learning, and transfer learning. In Section 3 we formally state the policy transfer problem. In Section 4 we present \algoName and discuss algorithmic design features. In Section 5 we prove the accrued reward for \algoName is lower bounded. In Section 6 we present experimental results on a simulated car environment and a two-link robotic manipulator, as well as present results for \algoName with robust policy learning methods. Finally, in Section 7 we draw conclusions and discuss future directions.

\section{Related Work and Background}

A plethora of work in both learning and control theory has addressed the problem of varying system dynamics, especially in the context of safe policy transfer and robust control. 


\runinhead{Transfer in reinforcement learning} The problem of high sample complexity in reinforcement learning has generated considerable interest in policy transfer.
Taylor et al. provide an excellent review of approaches to the transfer learning problem~\cite{Taylor2009TransferLF}.
A series of approaches focused on reducing the number of rollouts performed on a physical robot, by alternating between policy improvement in simulation and physical rollouts~\cite{abbeel2006using,levine2016end}. In those works, a time-dependent term is added to the dynamics after each physical rollout to account for unmodeled error. This approach, however, does not address robustness in the initial transfer, and the system could sustain or cause damage before the online learning model converges. 

The \textsc{EPOpt} algorithm~\cite{rajeswaran2016epopt} randomly samples dynamics parameters from a Gaussian distribution prior to each training run, and optimizes the reward for the worst-performing $\epsilon$-fraction of dynamics parameters. However, it is not clear how robust it is against disturbances not explicitly experienced in training. This approach is conceptually similar to that in \cite{mordatch2015ensemble}, in which more traditional trajectory optimization methods are used with an ensemble of models to increase robustness. Similarly, \cite{mandlekar2017arpl} and \cite{pinto2017robust} use adversarial disturbances instead of random dynamics parameters for robust policy training. Tobin et al. \cite{tobin2017domain} and Peng et al. \cite{peng2017sim} randomize visual inputs and dynamics parameters respectively. Bousmalis et al. \cite{bousmalis2017using} meanwhile adapt rendered visual inputs to reality using a framework based on generative adversarial networks, as opposed to strictly randomizing them. While this may improve adaptation to a target environment in which these parameters are varied, this may not improve performance on dynamics changes outside of those varied; in effect, it does not mitigate errors due to the ``unknown unknowns''. 

Christiano et al. \cite{christiano2016transfer} approach the transfer problem by training an inverse dynamics model on the target system and generating a nominal trajectory of states. The inverse dynamics model then generates actions to connect these states. However, there are no guarantees that an action exists in the target dynamics to connect two learned adjacent states. Moreover, this requires training on the target environment; in this work we consider zero-shot learning where this is not possible. Recently, the problem of transfer has been addressed in part by rapid test adaptation~\cite{devin2016learning,rusu2016progressive}. These approaches have focused on training modular networks that have both ``task-specific'' and ``robot-specific'' modules. This then allows the task-specific module to be efficiently swapped out and retrained. However, it is unclear how error in the learned model affects these methods. 

In this work we aim to perform zero-shot policy transfer, and thus efficient model-based approaches are not directly applicable. However, our approach uses an auxiliary control scheme that leverages model learning for an approximate dynamics model. When online learning is possible, sample-efficient model-based reinforcement learning approaches can dramatically improve sample complexity, largely by leveraging tools from planning and optimal control~\cite{kober2013reinforcement}. However, these models require an accurate estimate of the true system dynamics in order to learn an effective policy. A variety of model classes have been used to represent system dynamics, such as neural networks~\cite{heess2015learning}, Gaussian processes~\cite{deisenroth2011pilco}, and local linear models~\cite{gu2016continuous,levine2016end}. 

\runinhead{Robust control}
Trajectory optimization methods have been widely used for robotic control~\cite{tassa2012synthesis}. Among these optimization methods, model predictive control (MPC) is a class of online methods that perform trajectory optimization in a receding-horizon fashion~\cite{neunert2016fast}. This receding-horizon approach, in which a finite-horizon, open-loop trajectory optimization problem is continuously re-solved, results in an online control algorithm that is robust to disturbances. Several works have attempted to combine trajectory optimization methods with dynamics learning~\cite{mitrovic2010adaptive} and policy learning~\cite{kahn2016plato}. In this work, we develop an auxiliary robust MPC-based controller to guarantee robustness and performance for learned policies. Our method combines the strengths of deep policy networks~\cite{schulman2015trust} and tube-based MPC~\cite{mayne2011tube} to offer a controller with good performance as well as robustness guarantees.  




\section{Problem Setup and Preliminaries}

Consider a finite-horizon Markov Decision Process ($\mathcal{M}$) defined as a tuple $\mathcal{M}:\langle \mathcal{S},\mathcal{A}, p, r, T\rangle$. Here $\mathcal{S}$ and $\mathcal{A}$ represent continuous, bounded state and action spaces for the agent, 
$r: \mathcal{S}\times \mathcal{A} \rightarrow \mathbb{R}$ is the reward function that maps a state-action tuple to a scalar, and $T$ is the problem horizon.
Finally, $p : \mathcal{S} \times \mathcal{S} \times \mathcal{A} \to [0,1]$ is the transition distribution that captures the state transition dynamics in the environment and is a distribution over states conditioned on the previous state and action. The goal is to find a policy $\pi: \mathcal{S}\rightarrow \mathcal{A}$ that maximizes the expected cumulative reward over the choice of policy:
\begin{equation}
\begin{aligned}
& \pi^*(s) = \underset{\pi(s)}{\text{argmax}} \: \mathbb{E}\left[ \sum_{t=0}^{T} r(s_t,a_t) \right].
\end{aligned}
\end{equation}

The above reflects a standard setup for policy optimization in continuous state and action spaces. In this work, we are interested in the case in which we only have an approximately correct environment, which we refer to as the source environment (e.g. a physics simulator). We may sample this simulator an unlimited number of times, but we wish to maximize performance on the first execution in a target environment. Without any assumptions on the correctness of the simulator, this problem is of course intractable as the two sets of dynamics may be arbitrarily different. However, relatively loose assumptions about the correctness of the simulator are very reasonable, based on the modeling fidelity of the simulator. We assume the simulator (denoted $\mathcal{M}_{S}$) has deterministic, twice continuously-differentiable dynamics $s_{t+1} = f(s_t,a_t)$. Then, let the dynamics of the target environment (denoted $\mathcal{M}_{T}$) be denoted $s_{t+1} = f(s_t,a_t) + w_t$, for iid additive noise $w_t$ with compact, convex support $\mathcal{W}$ that contains the origin. Generally, the noise distribution may be state and action dependent, so this formulation reduces to standard formulations in both robust and stochastic control \cite{zhou1996robust}. We assume all other components of the MDPs defining the source and target environments are the same (e.g. reward function). Finally, we assume the reward function $r$ is Lipschitz continuous, an assumption that we discuss in more detail in section 5.
Based on the above definitions, we can now state the problem we aim to solve. 

\runinhead{Problem Statement} Given the simulator dynamics and the problem defined by the MDP $\mathcal{M}_S$, we wish to learn a policy to maximize the reward accrued during operation in the target system, $\mathcal{M}_T$. Formally, if we write the realization of the disturbance at time $t$ as $\tilde{w}_t$, we wish to solve the problem:
\begin{equation}
\begin{aligned}
&\underset{\{a_t\}_{t=0}^T}{\text{max}}\,\,\,\,\mathbb{E}\left[ \sum_{t=0}^{T} r(s_t,a_t) \right] \quad\\
&\,\,\,\text{s.t.}\,\,\quad s_{t+1} = f(s_t,a_t) + \tilde{w}_t, \, \text{and}\,\, s_t \in \mathcal{S},\, a_t \in  \mathcal{A} \quad \forall\, t \in [0, T],
\end{aligned}
\end{equation}
while only having access to the simulator, $\mathcal{M_S}$, for training. 

\section{\algoName: \algoFull}
\seclabel{adapt}

In this section we present the \algoName algorithm for zero-shot transfer. A high level view of the algorithm is presented in Algorithm~\ref{alg:overview}. First, we assume that a policy is trained in simulation. Our approach is to first compute a nominal trajectory (without disturbance) by continuously executing the learned policy on the simulator dynamics. Then, when transferred to the target environment, we use an auxiliary model predictive control-based (MPC) controller to stabilize around this nominal trajectory. In this work, we use a reward formulation for operation in the primary environment (i.e, the aim is to maximize reward), and a cost formulation for the auxiliary controller (i.e., the aim is to minimize cost to thus minimize deviation from the nominal trajectory). This is in part to disambiguate the distinction between the primary and auxiliary optimization problems. 

\runinhead{Policy Training} We use model-free policy optimization on the black-box simulated model. Our theoretical guarantees rely on the auxiliary controller avoiding saturation. Therefore, if a policy operates near the limits of its control authority and thus the auxiliary controller saturates when used on the target environment, this policy is trained using restricted state and action spaces $\mathcal{S}' \subseteq \mathcal{S}$, $\mathcal{A}' \subseteq \mathcal{A}$. We let $\mathcal{M'}$ denote an MDP with restricted state and action spaces. This follows the approach of \cite{mayne2011tube}, where it is used to prevent auxiliary controller saturation. Intuitively, restricting the state and action space ensures any nominal trajectory in those spaces can be stabilized by the auxiliary controller. Therefore, if saturation is rare, restricting these sets is unnecessary. 

\algoName is invariant to the choice of policy optimization method. 
During online operation, a nominal trajectory $\tau = \{(\bar{s}_t, \bar{a}_t)\}_{t=0}^T$ is generated by rolling out the policy on the simulator dynamics, $\mathcal{M_S}$. The auxiliary controller then tracks this trajectory in the target environment. 

\runinhead{Approximate Dynamics Model} Because the model of the simulator is treated as a black-box, it is impractical to use for the auxiliary controller in an optimal control framework. As such, we rely on an approximate model of the dynamics, separate from the simulator dynamics $f$, which we refer to as $\hat{f}$. The specific representation of the model (e.g. linear model, feedforward neural network, etc.) depends on both the accuracy required as well as the method used to solve the auxiliary control problem. This model may be either learned from the simulator, or based on prior knowledge. A substantial body of literature exists on dynamics model learning from black-box systems \cite{moerland2017learning}. Alternatively, this model may be based on external knowledge, either from learning a dynamics model in advance from the target system or from, for example, a physical model of the system. 

\runinhead{Auxiliary MPC Controller} Our auxiliary nonlinear MPC controller is based on that of \cite{mayne2011tube}. Specifically, we write the auxiliary control problem:
\begin{equation}
\begin{aligned}
& \underset{\{a_k\}_{k=t}^{t+N}}{\text{min}}\quad \sum_{k=t}^{t+N} (s_k - \bar{s}_k)^T Q_k (s_k - \bar{s}_k) + (a_k - \bar{a}_k)^T R_k (a_k - \bar{a}_k)  \\
&\,\,\,\,\,\text{s.t.}\,\quad\quad s_{k+1} = \hat{f}(s_k,a_k), \, \text{and}\,\, s_k \in \mathcal{S},\, a_k \in  \mathcal{A} \quad \forall\, k \in [t,t+N],
\end{aligned}
\label{mpc-primary}
\end{equation}
where $N$ is the MPC horizon, $Q_k$ and $R_k$ are positive definite cost matrices for the state deviation and control deviation respectively, and $\hat{f}$ is the approximate dynamics model. In some cases, this problem is convex, but generally it may not be. In our experiments, this optimization problem is solved with iterative relinearization based on \cite{todorov2005generalized}. However, whereas they iteratively linearize the nonlinear optimal control problem and solve an LQR problem over the full horizon of the problem, we explicitly solve the problem over the MPC horizon. We do not consider terminal state costs or constraints. This formulation of the auxiliary controller by \cite{mayne2011tube} allows us to guarantee, under our assumptions, that our true state stays in a tube around the nominal trajectory, where the tube is defined by level sets of the value function (the details of this are addressed in Section 5). 

\begin{algorithm}[t]
\caption{\algoFull\,(\algoName)  \label{alg:overview}}
\begin{algorithmic}[1]
    \REQUIRE Source Env: $\mathcal{M}_S$, Target Env: $\mathcal{M}_T$, Initial State: $s_0$
    
    Offline:
    
    \STATE $\mathcal{A}', \mathcal{S}' \gets \texttt{bound\_set}(\mathcal{A}, \mathcal{S})$
    \hfill \COMMENT{\footnotesize Calculate constrained state \& action space} 
    
    \STATE $\pi \gets \texttt{policy\_opt} \big(\mathcal{M}'_S\big)$ \hfill \COMMENT{\footnotesize Train a policy for $\mathcal{M}'_S$ using constrained $\mathcal{S}',\mathcal{A}'$}
    
    \STATE $\hat{f} \gets \texttt{fit\_dynamics}\big(\mathcal{M}_S\big)$ \hfill \COMMENT{\footnotesize Fit Dynamics for $\mathcal{M}_S$}
    
    Online:
    
    \STATE $\tau \gets \texttt{rollout}\big(s_0, \pi, \mathcal{M}_S, T\big)$ \hfill \COMMENT{\footnotesize Roll out $\pi$ on $\mathcal{M}_S$ to get nominal trajectory}

    \STATE $s \gets s_0$
    
    \FOR{$t \in [0, T]$}
            \STATE $a \gets \texttt{aux\_MPC}\big(s, \tau, \hat{f}, \tau, N\big)$ \hfill \COMMENT{\footnotesize NMPC with iterative linearization}          \\            
            \STATE $s \gets f(s, a) + w$ 
            \hfill \COMMENT{\footnotesize Rollout the first step of action seq. on $\mathcal{M}_T$}    
    \ENDFOR
    
\end{algorithmic}
\end{algorithm}
\setlength{\textfloatsep}{5pt}

The solution to the MPC problem is iterative. First, we linearize around the nominal trajectory $\tau$. We introduce the notation $\{(\hat{s}_{k}, \hat{a}_k)\}_{k=t}^{k=t+N}$, which is the solution for the last iteration. These are initialized as $\hat{s}_{t} \gets \bar{s}_{t}$ and $\hat{a}_{t} \gets \bar{a}_{t}$. Then, we introduce the deviations from this solution as 
\begin{equation}
\begin{aligned}
\delta s_{t} &= s_{t} - \hat{s}_{t}, &\quad \delta a_{t} &= a_{t} - \hat{a}_{t}.
\end{aligned}
\end{equation}
Then, taking the linearization of our dynamics 
\begin{equation}
A_t = \Eval{\frac{\partial \hat{f}}{\partial s_t}}{{s_t = \hat{s}_t, a_t = \hat{a}_t}}{} \quad
B_t = \Eval{\frac{\partial \hat{f}}{\partial a_t}}{{s_t = \hat{s}_t, a_t = \hat{a}_t}}{}, 
\end{equation}
we can rewrite the MPC problem as:
\begin{equation}
\begin{aligned}
& \underset{\{\delta a_k\}_{k=t}^{t+N}}{\text{min}}\quad \sum_{k=t}^{t+N} (\delta s_k + \hat{s}_k- \bar{s}_k)^T Q_k (\delta s_k + \hat{s}_k - \bar{s}_k) + (\delta a_k + \hat{a}_k - \bar{a}_k)^T R_k (\delta a_k + \hat{a}_k - \bar{a}_k)  \\
&\,\,\,\,\,\,\text{s.t.}\,\quad\quad \delta s_{k+1} = A_k \delta s_k + B_k \delta a_k, \, \text{and}\,\, \delta {s}_k + \hat{s}_k \in \mathcal{S},\, \delta{a}_k +\hat{a}_k \in  \mathcal{A}, \quad \forall\, k \in [t, t+N].
\end{aligned}
\label{mpc-final}
\end{equation}
Note that the optimization is over the action deviations $\{\delta a_k\}_{k=t}^{t+N}$. Once this problem is solved, we use the update rule $\hat{s}_{t} \gets \hat{s}_{t} + \delta s_t$, $\hat{a}_{t} \gets \hat{a}_{t} + \delta a_t$. The dynamics are then relinearized, and this is iterated until convergence. Because we use iterative linearization to solve the nonlinear program, it is necessary to choose a dynamics representation $\hat{f}$ that is efficiently linearizable. In our experiments, we use an analytical nonlinear dynamics representation for which the linearization can be computed analytically (see \cite{webb2013kinodynamic} for details), as well as fit a time-varying linear model. Choices such as, e.g., a Gaussian process representation, may be expensive to linearize.
\section{\algoName: Analysis}
\label{analysis}

The following section develops the main theoretical analysis of this study. We will first show that \algoName results in bounded deviation from the nominal trajectory $\tau$ under a set of technical assumptions. This result is then used to show that the deviation between cumulative reward of the realized rollout on the target system and the cumulative reward of the nominal trajectory on the source environment, is upper bounded. This is to say, the decrease in performance below the ideal case is bounded.

\subsection{Safety Analysis in \algoName}

Using the notation from Eq~\eqref{mpc-primary}, let us denote the solution at time $k$ as $C_N^*(s_k,k)$ for MPC horizon $N$. This is the minimum cost associated with the finite horizon problem that is solved iteratively in the MPC framework. Note that this problem is solved with the approximate dynamics model; in the case where the approximate dynamics model exactly matches the target environment model, the solution to this problem would have value zero as the trajectory would be tracked exactly. We denote by $\kappa_N(s_k)$ the action at time $k$ from the solution to the MPC problem. Then, let $L_d(k) \triangleq \{s \mid C_N^*(s,k) \leq d \}$ denote the level set of the cost function for some value $d \in (0,c)$ (for some constant $c$; see \cite{mayne2011tube}) at time $k$.

We assume the error between approximate dynamics representation $\hat{f}$ and the simulator dynamics $f$ is outer approximated by a compact, convex set $\mathcal{D}$ that contains the origin. Therefore, for all state, action pairs $(s,a) \in \mathcal{S} \times \mathcal{A}$, $f(s,a) - \hat{f}(s,a) \in \mathcal{D}$. In the case where the state and action spaces are bounded, there always exists an outer approximation which satisfies this assumption. However, in practice, it is likely considerably smaller than this worst case. 

Let $\mathcal{T}_s(s_0) \triangleq \{L_d(k) \mid k \in \mathbb{Z}_{\geq0} \}$ denote a state tube defined by the time-dependent level sets of the auxiliary cost function. We may now state our first result, noting that the auxiliary stabilizing policy $\kappa_N$ is the result of the MPC optimization problem relying solely on the approximate dynamics $\hat{f}$.

\begin{theorem}
Every state trajectory $\{s_t\}_{t=0}^T$ generated by the target dynamics $s_{t+1} = f(s_t,\kappa_N(s_t)) + w_t$ with initial state $s_0$, lies in the state tube $\mathcal{T}_s(s_0)$.
\end{theorem}

\begin{proof}
Note that $\mathcal{W} + \mathcal{D}$, where the addition denotes a Minkowski sum, is compact, convex, and contains the origin. Then, the result follows from Theorem 1 of \cite{mayne2011tube} by replacing the set of disturbances (which the authors refer to as $\mathbb{W}$) with $\mathcal{W} + \mathcal{D}$.
\end{proof}

The above result combined with Proposition 2i of \cite{mayne2011tube}, which shows that for some constant $c_1$, $C^*_N(s_k,k) \geq c_1 \| s_k - \bar{s}_k \|^2$, gives insight into the safety of \algoName. In particular, note that for an arbitrarily long trajectory, the realized trajectory stays in a region around the nominal trajectory despite using an inaccurate dynamics representation in the MPC optimization problem. While this result shows that the deviation from the nominal trajectory is bounded, it does not allow construction of explicit tubes in the state space, and thus can not be used directly for guarantees on obstacle avoidance. Recent work by Singh et al. \cite{singh2017robust} establishes tubes of this form, and this is thus a promising extension of the \algoName framework. 

\subsection{Robustness Analysis in \algoName}

We will now show that due to the boundedness of state deviation, the deviation in the total accrued reward over a rollout on the target system is bounded. Let $V_S^\pi(s)$ and $V_T^\kappa(s)$ denote the value functions associated with some state $s$ and the primary policy executed on the source environment, and the \algoName policy on the secondary environment respectively. 

\begin{theorem}
Under the technical assumptions made in Section 3 and 5.1, $|V_T^{\kappa}(s_0) - V_S^\pi(s_0)| \leq c_2 \sum_{t=0}^T \sqrt{C^*_N(s_t,t)}$, where $c_2$ is some constant and $s_{t+1} = f(s_t, \kappa_N(s_t)) + w_t$. 
\end{theorem}

\begin{proof}
First, note $|V_{\kappa}(s_0) - V_\pi(s_0)| \leq \sum^T_{t=0} |r(s_t,\kappa_N(s_t)) - r(\bar{s}_t, \pi(\bar{s}_t))|$, where $s_{t+1} = f(s_t,\kappa_N(s_t)) + w_t$ and $\bar{s}_{t+1} = f(\bar{s}_t,\pi(\bar{s}_t))$. Additionally, letting $a = \kappa_N(s)$ and $\bar{a} = \pi(\bar{s})$, note that similarly to Proposition 2i of \cite{mayne2011tube}, we can establish a bound on the action deviation from the nominal trajectory in terms of the auxiliary cost function, $C^*(s_t,t) \geq c_3 \|a_t - \bar{a}_t\|^2$ for all $t$ (where the norm is in the Euclidean sense), by taking $c_3$ as the minimum eigenvalue of $R_t$. By the Lipschitz continuity of the reward function, and writing the Lipschitz constant of the reward function $L_r$, we have 
\begin{equation}
    \label{rew_bound}
    |r(s,a) - r(\bar{s},\bar{a})| \leq L_r ( \|s - \bar{s}\| + \|a - \bar{a}\| ).
\end{equation}
Then, noting that the quadratic auxiliary cost function $C^*_N$ is always positive, the result is proved by applying Proposition 2i of \cite{mayne2011tube} and the bound on action deviation from the nominal to the right hand side of Equation \ref{rew_bound}.
\end{proof}

This result may then be restated in terms of the disturbance sets. Let $\|\mathcal{W} + \mathcal{D}\| \triangleq \max_{w \in \mathcal{W},\, d \in \mathcal{D}} \|w + d\|$.

\begin{theorem}
Under the same technical assumptions as Theorem 2, the following inequality holds for some constant $c_4 > 0$:
\begin{equation}
    |V_T^{\kappa}(s_0) - V_S^\pi(s_0)| \leq c_4 T \sqrt{\| \mathcal{W} + \mathcal{D}\|}
\end{equation}
\end{theorem}

\begin{proof}
The result follows from combining Theorem 2 with Proposition 4ii of \cite{mayne2011tube}.
\end{proof}


These results shows that along with guarantees on spatial deviation from the nominal trajectory, we may also establish bounds on the accrued reward relative to what is received with the nominal policy in the source environment, in effect demonstrating that zero-shot transfer is possible. The Lipschitz continuity of the reward function is essential to this result, and this illustrates several aspects of the policy transfer problem.

The \algoName algorithm is based on tracking a nominal rollout in simulation. Critical in the success of this approach is gradual variation of the reward function. Sparse reward structures are likely to fail with this approach to transfer, as tracking the nominal trajectory, even relatively closely, may result in poor reward. On the other hand, a slowly varying reward function, even if tracked relatively roughly may result in accrued reward close to the nominal rollout on the source environment. 

\section{Experimental Evaluation}

\begin{figure}[t]
    \vspace{-5pt}
        \centering  
    \begin{subfigure}[t]{0.5\textwidth}
        \centering
        \hspace{2.35mm}
        \includegraphics[scale=0.076]{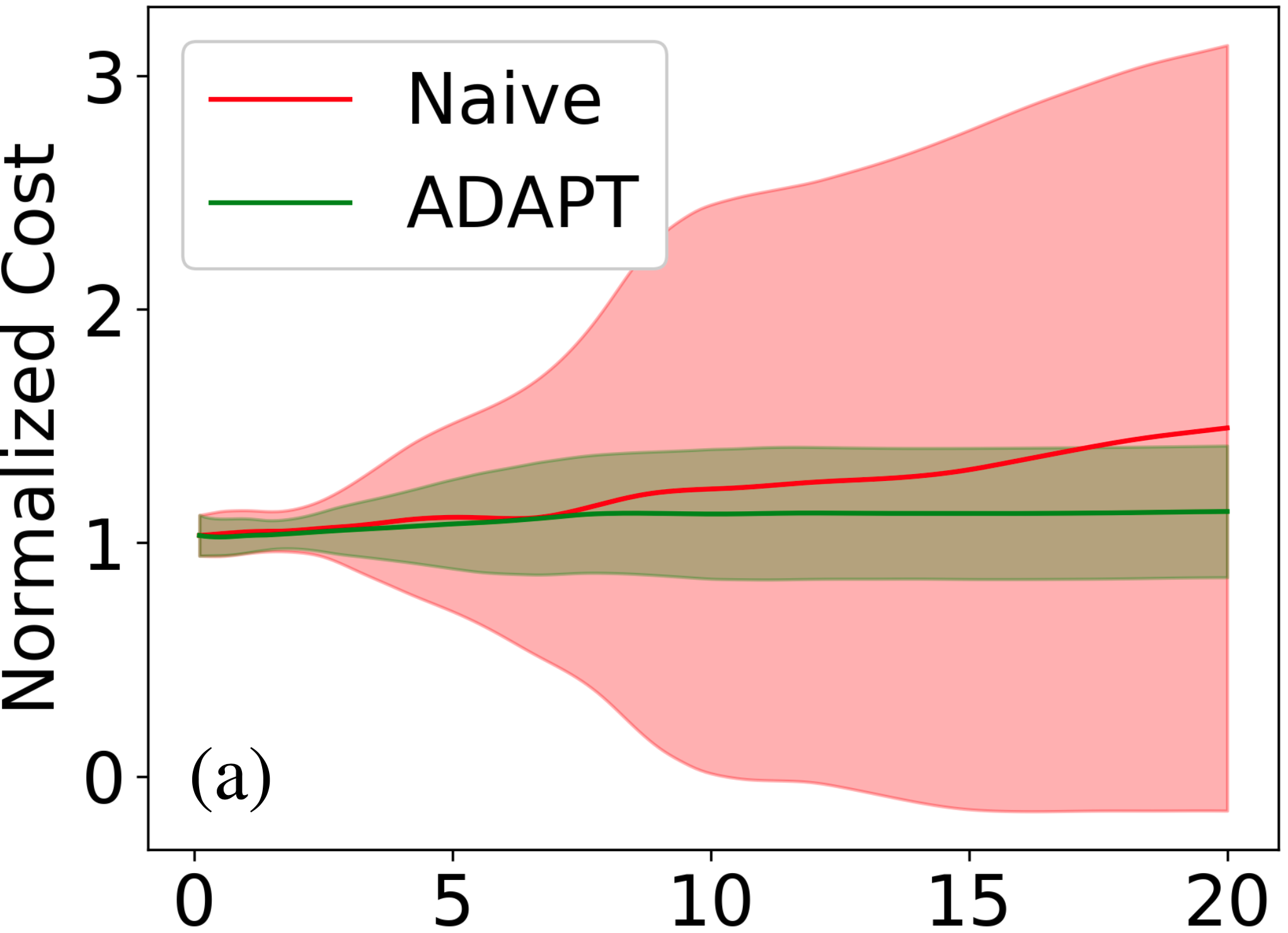}
    \end{subfigure}%
    \begin{subfigure}[t]{0.5\textwidth}
        \centering
        \hspace{-2.5mm}
        \includegraphics[scale=0.076]{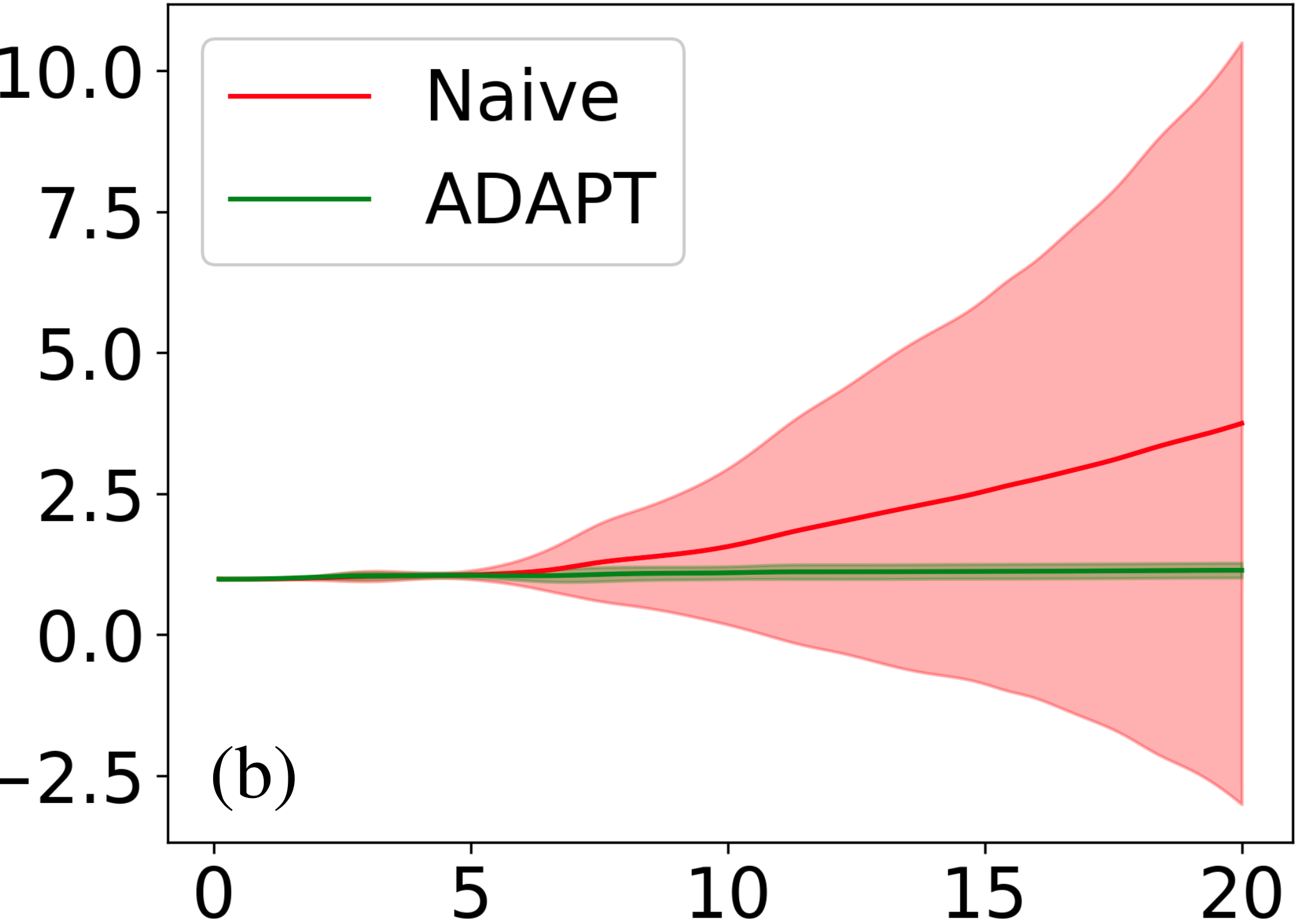}
    \end{subfigure}%
    \vspace{1mm}
    
    \begin{subfigure}[t]{0.5\textwidth}
        \centering
        \includegraphics[scale=0.076]{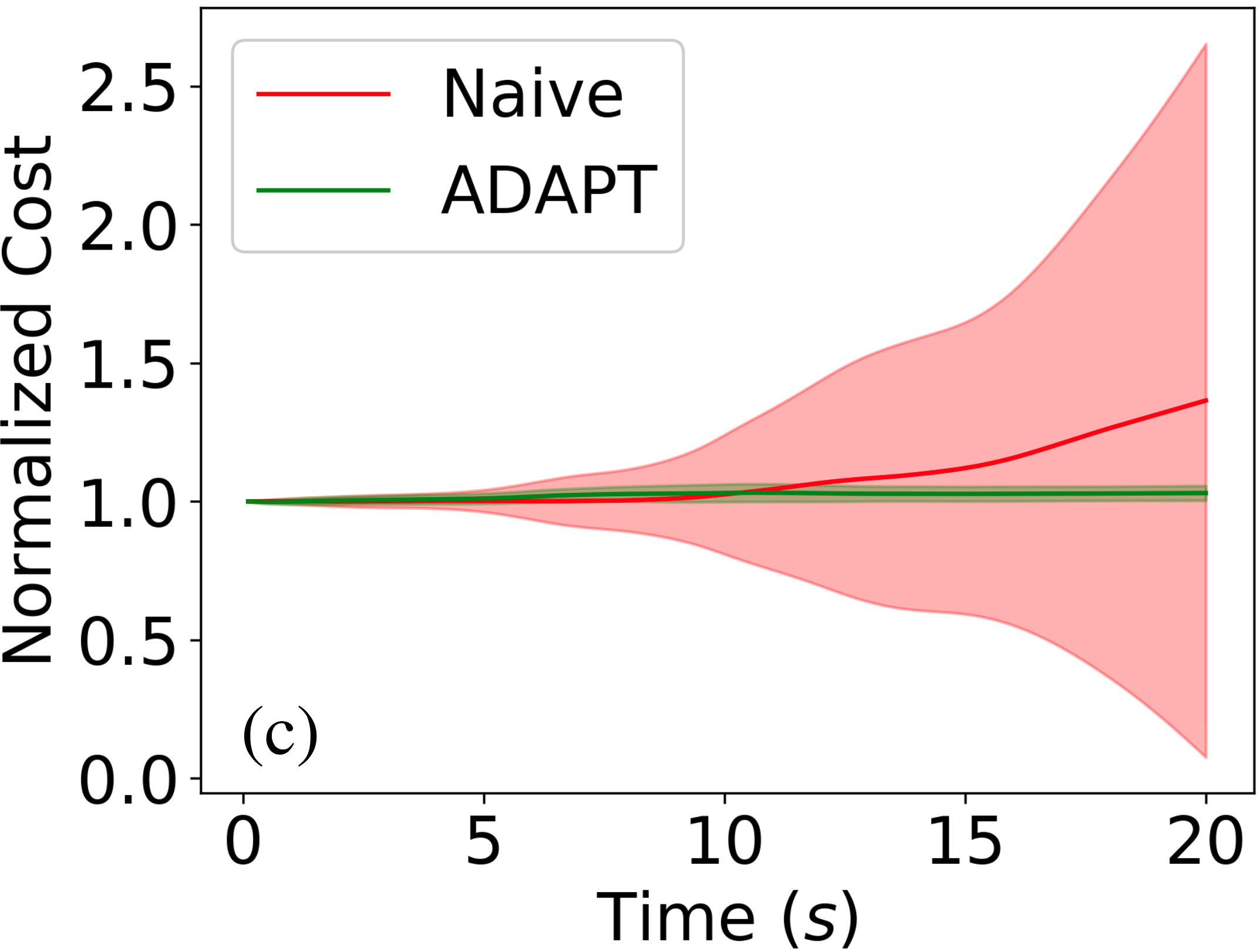}
    \end{subfigure}%
    \hspace{-0.8mm}
    \begin{subfigure}[t]{0.5\textwidth}
        \centering
        \includegraphics[scale=0.076]{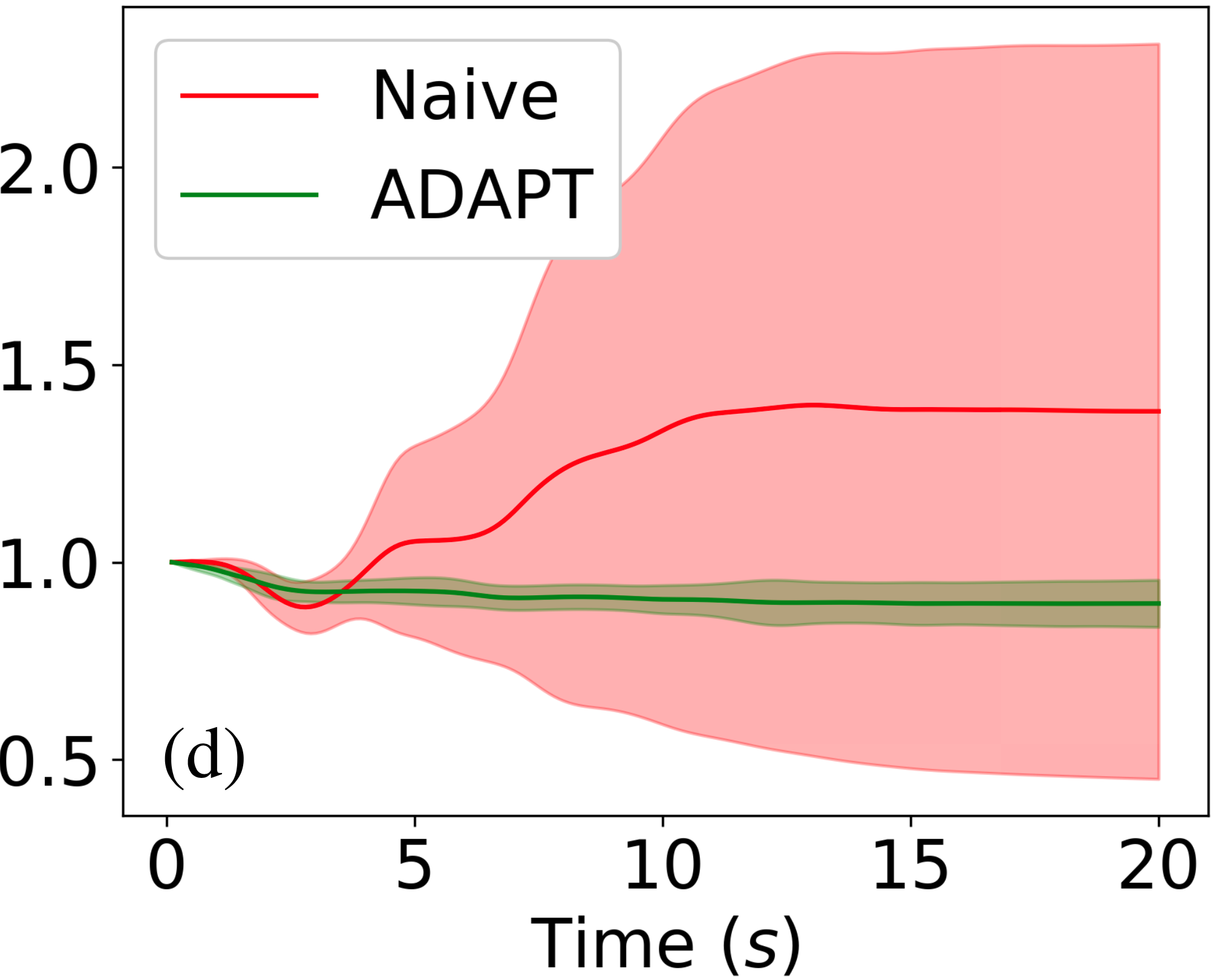}
    \end{subfigure}   
    \vspace{-10pt}
    \caption{Mean cumulative cost over the length of an episode for 50 episodes on the kinematic car environment. The confidence intervals are standard error. The costs are normalized to the cost of the naive policy being rolled out on the simulated environment from the same initial state, to allow more direct comparison across episodes. The \textit{naive} rollout is the nominal policy executed on the target environment. The disturbances tested are \textbf{a)} a hill landscape, \textbf{b)} additive control error, \textbf{c)} process noise, and \textbf{d)} dynamics parameter error.}
    
    \label{fig:car_cost}
\end{figure}

We implemented \algoName\ on a nonlinear, non-holonomic kinematic car model with a 5-dimensional state space as well as on the \texttt{Reacher} environment in OpenAI's Gym \cite{brockman2016}. 
We train policies using Trust Region Policy Optimization (TRPO) \cite{schulman2015trust}. The policy is parameterized as a neural network with two hidden layers, each with 64  units and ReLU nonlinearities.
In all of our experiments, we report \textit{normalized cost}. This is the cost (negative reward) realized by a trial in the target environment, divided by the cost of the nominal policy rolled out on the simulated environment from the same initial state. This allows more direct comparison between episodes for environments with stochastic initial states. We generally compare the \textit{naive} trial, which is the nominal policy rolled out on the target environment (e.g., standard transfer with no adaptation) to \algoName.

\subsection{Environment I: 5-D Car}

\begin{figure}[t]
    \vspace{-5pt}
    \centering
        
    \begin{subfigure}{0.31\textwidth}
        \centering
        \includegraphics[scale=0.092]{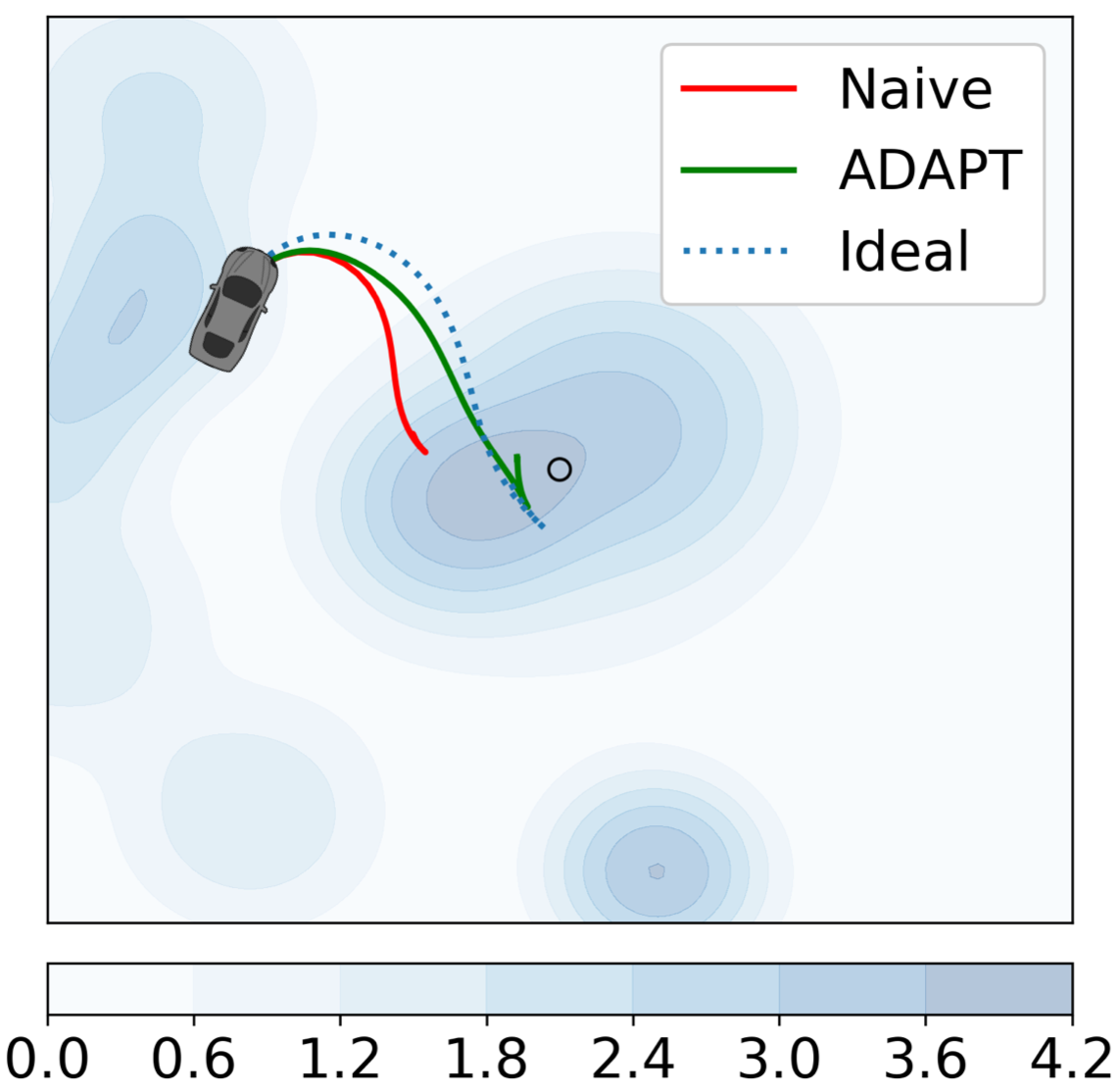}
        \caption{}
        
        \label{fig:hill}
    \end{subfigure}%
    \begin{subfigure}{0.33\textwidth}
        \centering
        \includegraphics[scale=0.3]{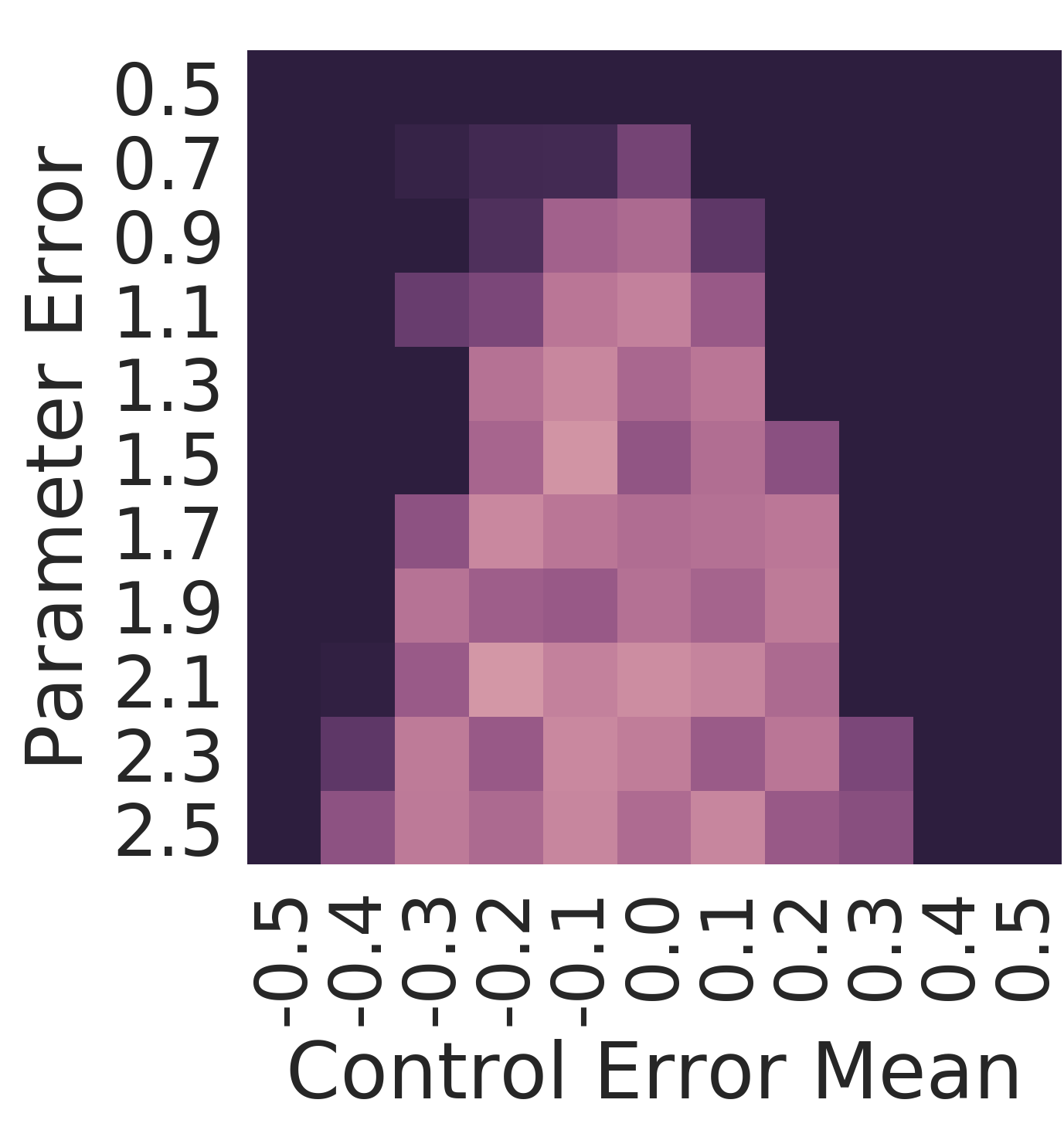}
        \caption{}
        
        \label{fig:heatmap_naive}
    \end{subfigure}%
    \begin{subfigure}{0.33\textwidth}
        \centering
        \includegraphics[scale=0.3]{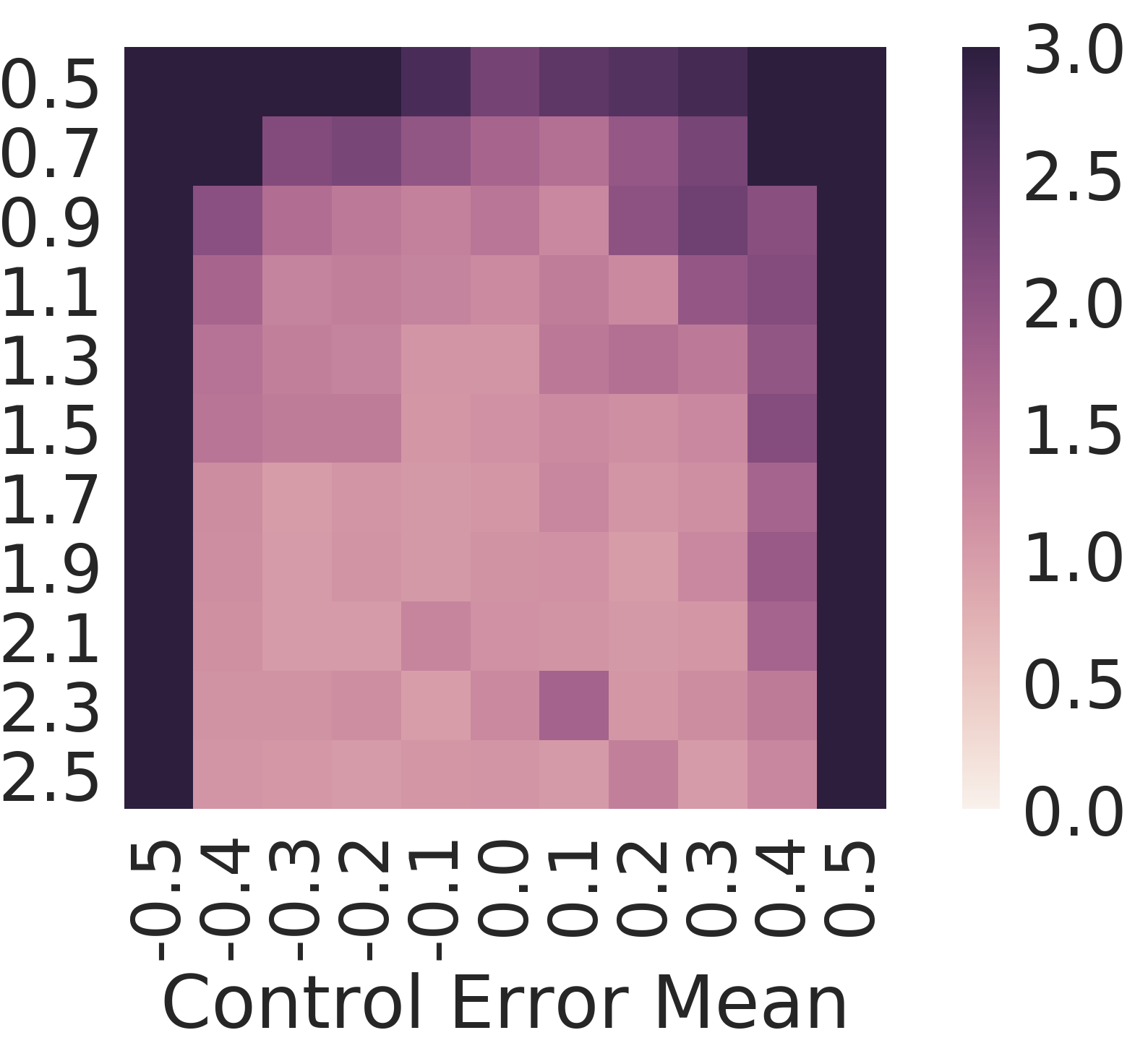}
        \caption{}
        \label{fig:heatmap_mpc}
    \end{subfigure}%
    \vspace{-10pt}
    \caption{\textbf{a)} The car environment with the paths for the ideal case (nominal policy on simulated environment), the naive case (nominal policy on the target environment), and the \algoName case (\algoName on the target environment). The contour plot shows the height of the added hills. Figures (b) and (c) show the normalized cost for varying disturbances due to additive control error and dynamics parameter error for \textbf{b)} the naive case and \textbf{c)} \algoName (lower is better). In addition to the listed disturbances, disturbances due to hills are also added for all trials. Each grid cell is the mean of 50 trials.}
\end{figure}

We implemented \algoName\ on a nonlinear, nonholonomic 5-dimensional kinematic car model that has been used previously in the motion planning literature \cite{webb2013kinodynamic}. Specifically, the car has state $s = \lbrack x, y, \theta, v, \kappa \rbrack^T$, where $x$ and $y$ denote coordinates in the plane, $\theta$ denotes heading angle, $v$ denotes speed, and $\kappa$ denotes trajectory curvature. The system has dynamics $\dot{s} = \lbrack v \cos \theta, v \sin \theta, v \kappa, a_v, a_\kappa \rbrack$, where $a_v \in \lbrack -2, 2 \rbrack$ and $a_\kappa \in \lbrack -0.5, 0.5 \rbrack$ are the controlled acceleration and curvature derivative.  The policy is trained to minimize the quadratic cost $L(\bm{s}, \bm{a}) = \sum^T_{t=0} \ell(s_t,a_t)$, where $\ell(s_t,a_t) = x_t^2 + y_t^2 + a_{v,t}^2 + a_{\kappa,t}^2$, which results in policies that drive to the origin. In each trial, the vehicle is initialized in a random state, with position $x,y \in [-5,5]$, with random heading and zero velocity and curvature. 

Our auxiliary controller used an MPC horizon of 2 seconds (20 timesteps). Our state deviation penalty matrix, $Q$, has value 1 along the diagonal for the position terms, and zero elsewhere. Thus, the MPC controller penalizes only deviation in position. The matrix $R$ had small terms ($10^{-3}$) along the diagonal to slightly penalize control deviations. In practice, this mostly acts as a small regularizing term to prevent large oscillatory control inputs by the auxiliary controller. The behavior of the auxiliary controller is dependent on the matrices $Q$ and $R$, but in practice good performance may be achieved across environments with fixed values. Because of the relatively high quadratic penalty on control in policy training, the nominal policy rarely approaches the control limits. Thus, we can set $\mathcal{A}' = \mathcal{A}$, and we set $\mathcal{S}' = \mathcal{S}$. For our dynamics model, we use the linearization reported in \cite{webb2013kinodynamic}.


\subsection{Disturbance Models}

We investigate four disturbance types:
\begin{enumerate}
    \item Environmental Uncertainty: We add randomly-generated hills to the target environment such that the car experiences accelerations due to gravity. This noise is therefore state-dependent. Figure \ref{fig:hill} shows a randomly generated landscape. We randomly sample 20 hills in the workspace, each of which is circular and has varying radius and height. The vehicle experiences an additive longitudinal acceleration proportional to the landscape slope at its current location, and no lateral acceleration.
    \item Control noise: Nonzero-mean additive control error drawn from a uniform distribution.
    \item Process noise: Additive, zero-mean noise added to the state. Disturbances are drawn from a uniform distribution. 
    \item Dynamics parameter error: We add a scaling factor $\gamma$ to the control of $\dot{\kappa}$, such that $\dot{\kappa} = \gamma a_\kappa$.
\end{enumerate}
For the last three, the noise terms were drawn i.i.d. from a uniform distribution at each time $t$. These disturbances were investigated both independently (Figure \ref{fig:car_cost}) and simultaneously (Figure 2). Figure 1 shows the normalized cost of the naive transfer and \algoName for each of the four disturbances individually. 

In our experiments, \algoName substantially outperforms naive transfer, achieving normalized costs 1.5-5x smaller. Additionally, the variance of the naive transfer is considerably higher, whereas the realized cost for \algoName is clustered relatively tightly around one (e.g., approximately equal cost to the ideal case). In Figure \ref{fig:car_cost}d, the normalized cost of \algoName is actually below one, implying that the transferred policy performs better than the ideal policy. In fact, this is because the dynamics parameter error in this trial results in oversteer, and so the agent accumulates less cost to turn to face the goal than in the nominal environment. Thus, pointing toward the goal is more ``cost-efficient'' in the target environment. The performance of direct transfer and \algoName with varying parameter error may be seen in Figure \ref{fig:heatmap_naive} and Figure \ref{fig:heatmap_mpc}. In Figure \ref{fig:hill}, a case is presented where the direct policy transfer fails to make it up a hill, whereas the \algoName policy tracks the nominal trajectory well. 

\subsection{\algoName with Robust Offline Policy}

\begin{figure}[t]
    \vspace{-5pt}
    \hspace{-1.5mm}
        \centering
    \begin{subfigure}[t]{0.5\textwidth}
        \hspace{2.45mm}
        \centering
        \includegraphics[scale=0.075]{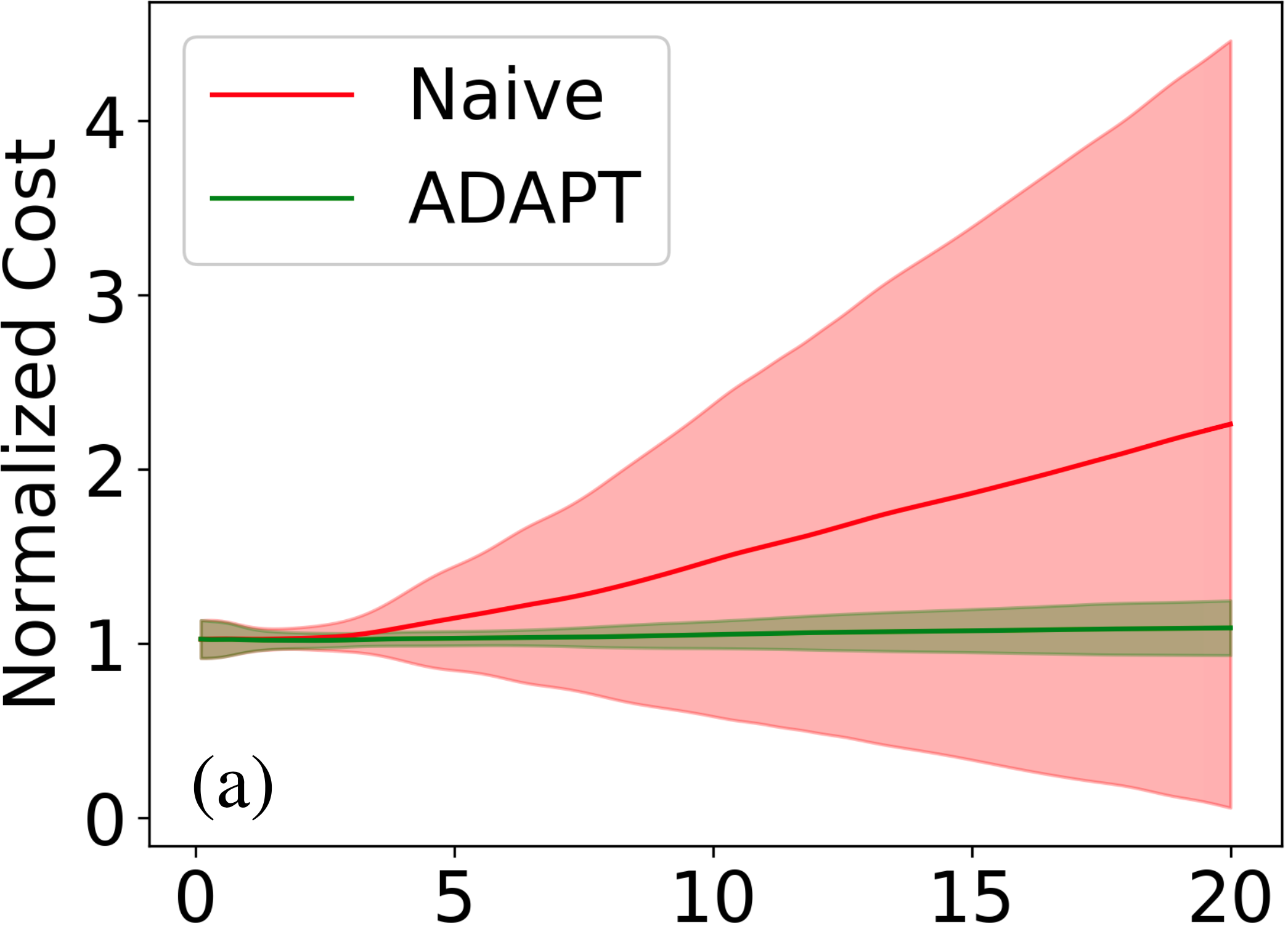}
    \end{subfigure}%
    \begin{subfigure}[t]{0.5\textwidth}
        \hspace{3.5mm}
        \centering
        \includegraphics[scale=0.075]{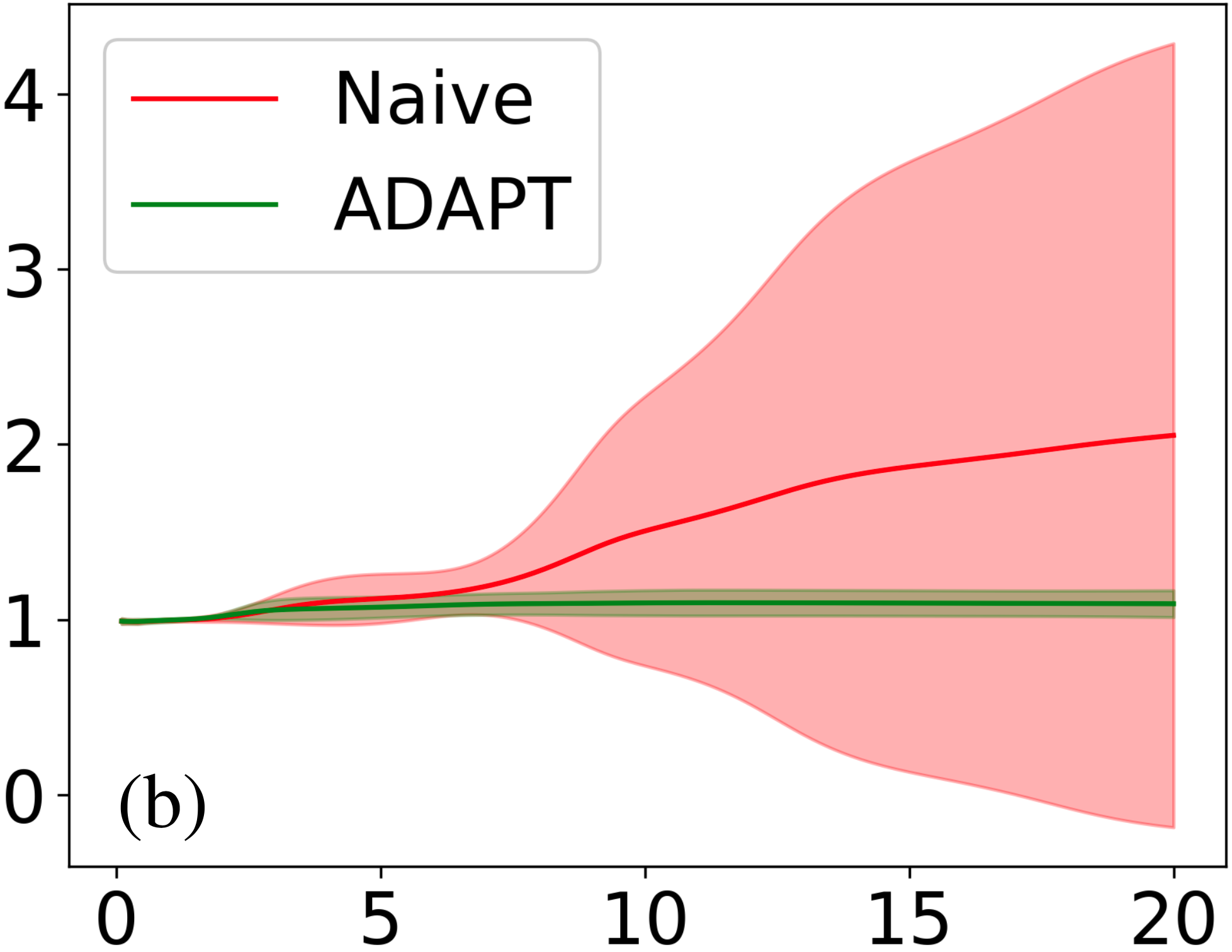}
    \end{subfigure}
    \hfill
    \begin{subfigure}[t]{0.5\textwidth}
        \centering
        \includegraphics[scale=0.075]{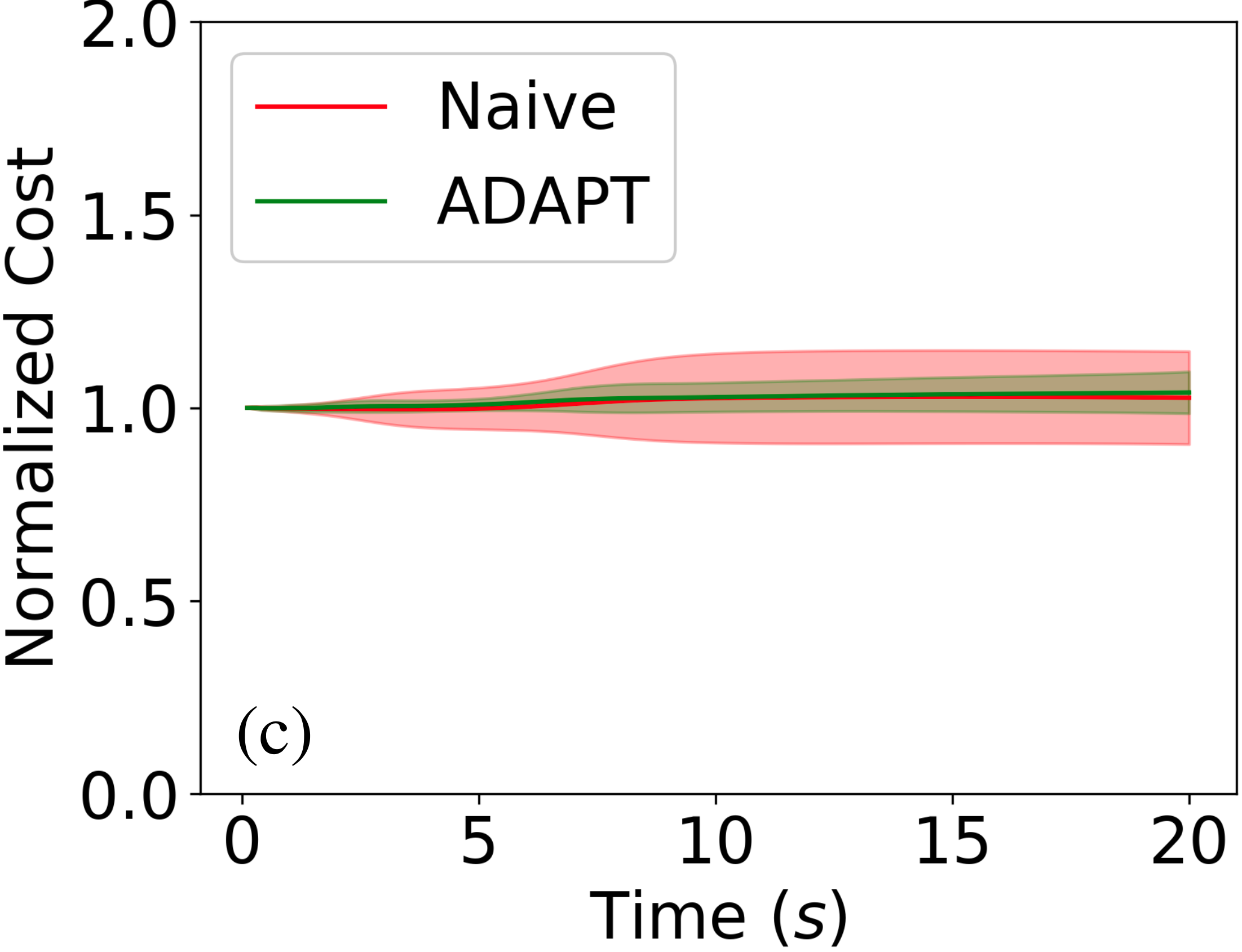}
    \end{subfigure}%
    \begin{subfigure}[t]{0.5\textwidth}
        \hspace{.5mm}
        \centering
        \includegraphics[scale=0.075]{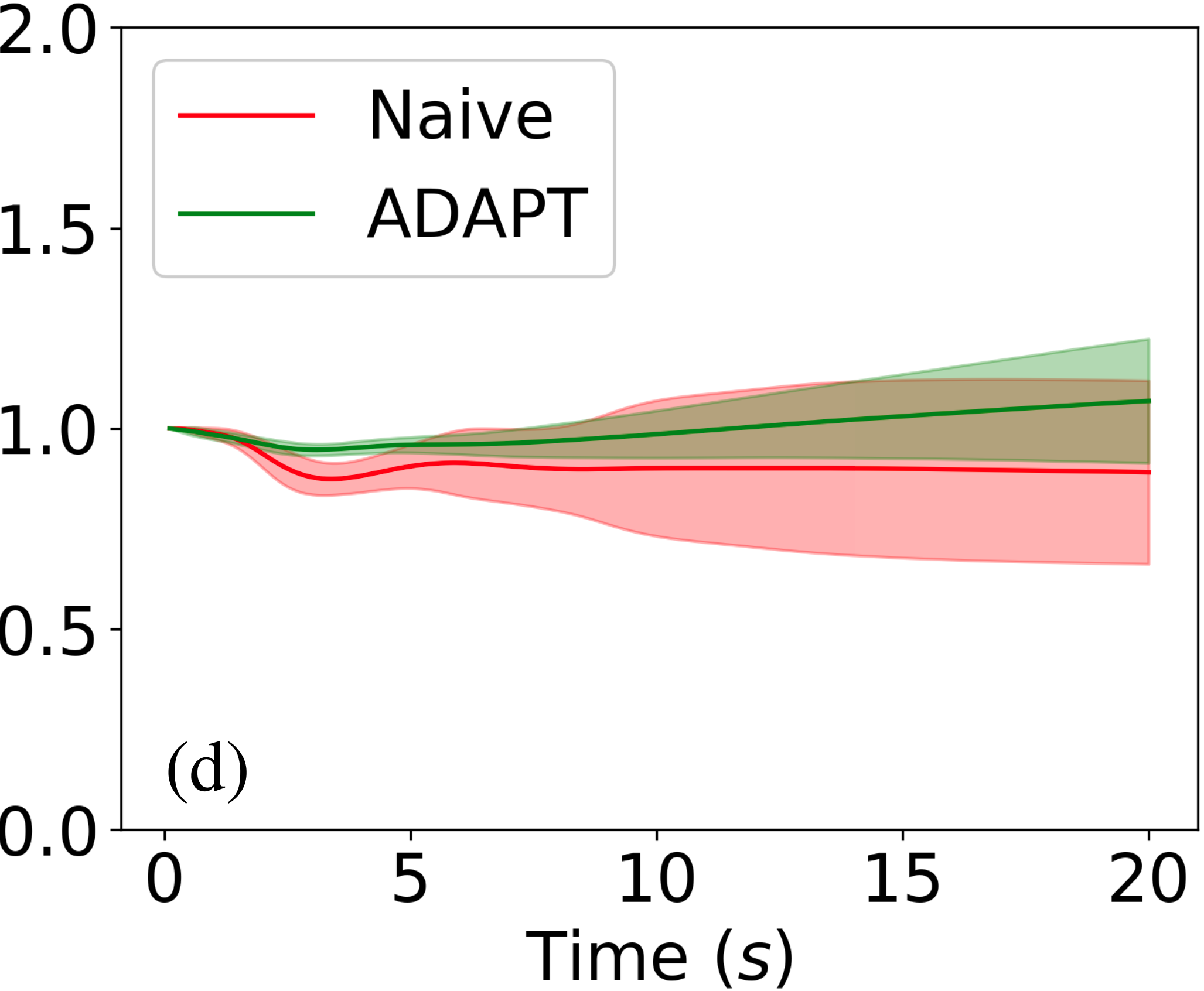}
    \end{subfigure} 
    \vspace{-10pt}
    \caption{Mean cumulative cost over the length of an episode for 50 episodes on the 5-D car environment, using an \textsc{EPOpt}-1 robust policy. The confidence intervals are standard error. The disturbances tested are \textbf{a)} a hill landscape, \textbf{b)} additive control error, \textbf{c)} process noise, and \textbf{d)} dynamics parameter error. The details of each noise source is presented in the supplementary materials.}
    \label{fig:car_cost_epopt}
\end{figure}

Whereas \algoName's approach to policy transfer relies primarily on stabilization in the target environment, recent work has focused on training robust policies in the source domain, and then performing direct transfer. In the \textsc{EPOpt} policy training framework \cite{rajeswaran2016epopt}, an agent is trained over a family of MDPs in which model parameters are drawn from distributions before each training rollout. Then, a Conditional Value-at-Risk (CVaR) objective function is optimized as opposed to an expectation over all training runs. We apply \algoName on top of an \textsc{EPOpt}-1 policy (equivalent to optimizing expected reward, with model parameters varying), and find that for disturbances explicitly varied during training, the performance of \textsc{EPOpt}-only transfer and \algoName are comparable. We add parameters $\gamma_i$ to the state derivative as follows: $\dot{s} = \lbrack \gamma_1 v \cos \gamma_2 \theta, \gamma_1 v \sin \gamma_2 \theta, \gamma_1  v \gamma_3 \kappa,\gamma_4 a_v,\gamma_5 a_\kappa \rbrack$. Each of these $\gamma_i$ are drawn from Gaussian distributions before each training run, and are fixed during the training run. Although  some of these parameters do not have a physical interpretation, the resulting policies are still robust to both parametric error, as well as process noise. In these experiments, an MPC horizon of 1 second was used (10 timesteps). The matrices $Q$ and $R$ were set as in Section 6.1.

In Figure \ref{fig:car_cost_epopt}, the comparison between the direct transfer of \textsc{EPOpt} policies and \algoName policies is presented. We can see that, for disturbances that are explicitly considered in training (specifically, model parameter error), naive transfer performs slightly better, albeit with higher variance. For other disturbances, like the addition of hills or control noise, \algoName significantly outperforms the directly-transferred policy. Indeed, while the performance of the \algoName policy is comparable to direct transfer for disturbances directly considered in training, unmodelled disturbances are handled substantially better by \algoName. Thus, to extract the best performance, we recommend applying the two approaches in tandem.

\subsection{Environment II: 2-Link Planar Robot Arm}

We next evaluate the performance of \algoName on the \texttt{Reacher} environment of Gym \cite{brockman2016}. This environment is a two link robotic arm that receives reward for proximity to a goal in the workspace, and is penalized for control effort. The state is a vector of the sin and cos of the joint angles, as well as joint angular velocities, the goal position, and the distance from the arm end-effector to the goal. In our tests, we fix one goal location and one starting state for all tests to more directly compare between trials. As such, the variance in normalized cost in experiments is much smaller than in the car experiments. For these experiments, the same noise models were used as in the previous section, with the exception of the ``hills'' disturbance. 

As an approximate dynamics model used for the auxiliary controller, we use the time-varying linear dynamics from \cite{levine2014learning}. This model is fit from rollouts in simulation. Since this model is linear, the MPC problem is convex, and the iterative MPC converges in one iteration. These dynamics are only valid in a local region, and thus must be fit for each desired policy rollout in the target environment. However, since the model is fit from simulation data, it is generated quickly and inexpensively. 

The results for normalized cost comparisons between naive transfer and \algoName are presented in Figure \ref{fig:reacher_cost}. We note that \algoName achieves significantly lower cost for additive control error and process noise, but achieves comparable cost for parameter error. The parameter varied in these experiments was the mass of the links of the arm. The effect of this change is to increase the inertia of the manipulator as a whole. In fact, this can be seen in the Figure 4c. While the cost of the naive transfer increases slowly, the cost of the \algoName trials spikes at approximately time $t = 0.25$. As \algoName is tracking the nominal trajectory, it increases the torque applied, thus suffering a penalty for the increased control action, but resulting in better tracking of the nominal trajectory. 

A similar effect can be observed in Figure 4a. The added control error actually drives the manipulator toward the goal, resulting in the dip in the normalized cost for both trajectories. However, the naive policy overshoots the goal substantially, and thus accrues substantially higher normalized cost than the \algoName experiments.

\begin{figure}[t]
\vspace{-5pt}
    \hspace{-1.5mm}
        \centering
    \begin{subfigure}[t]{0.33\textwidth}
        \centering
        \includegraphics[scale=0.054]{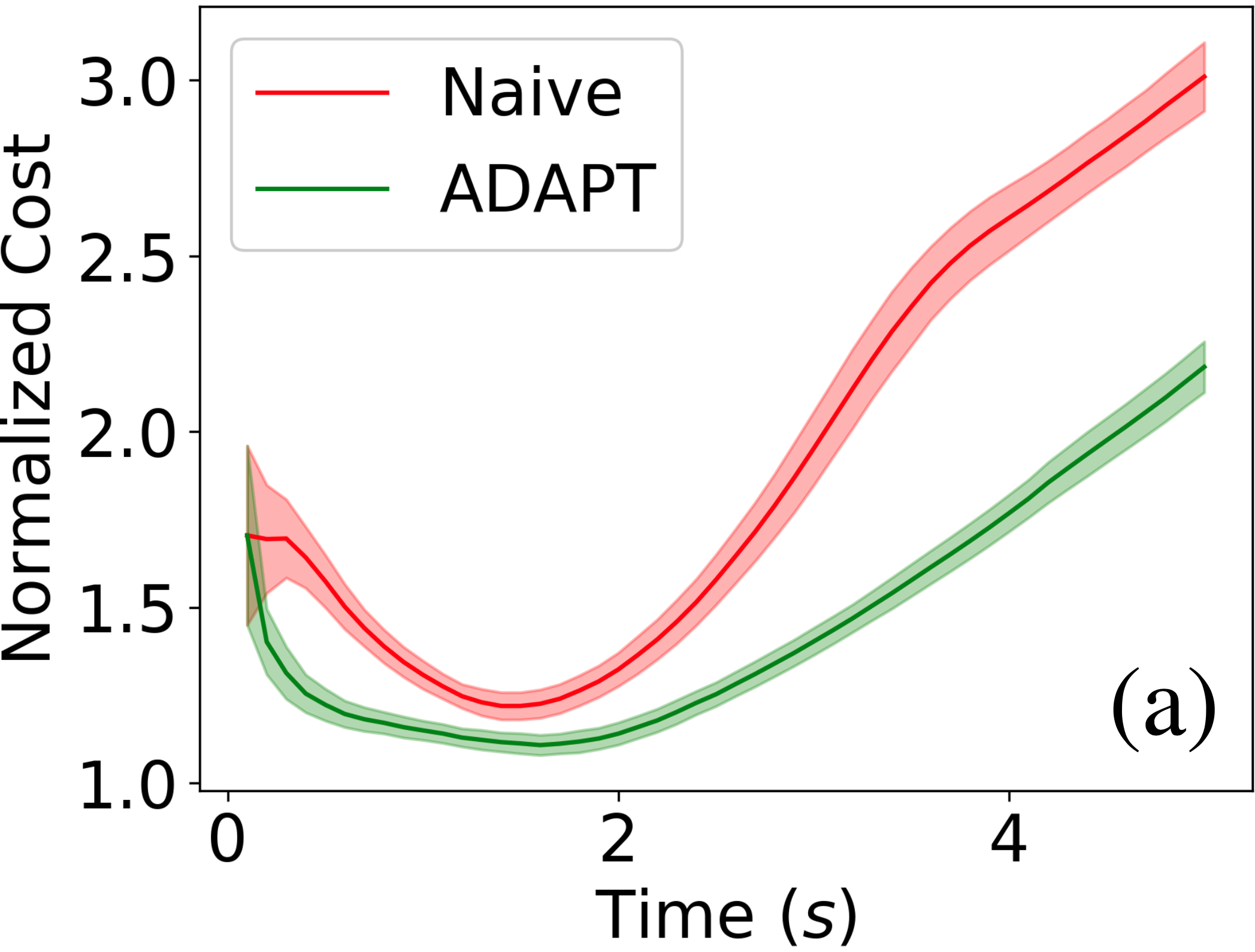}
        \label{fig:reacher_control}
    \end{subfigure}%
    \begin{subfigure}[t]{0.33\textwidth}
        \hspace{0.75mm}
        \centering
        \includegraphics[scale=0.054]{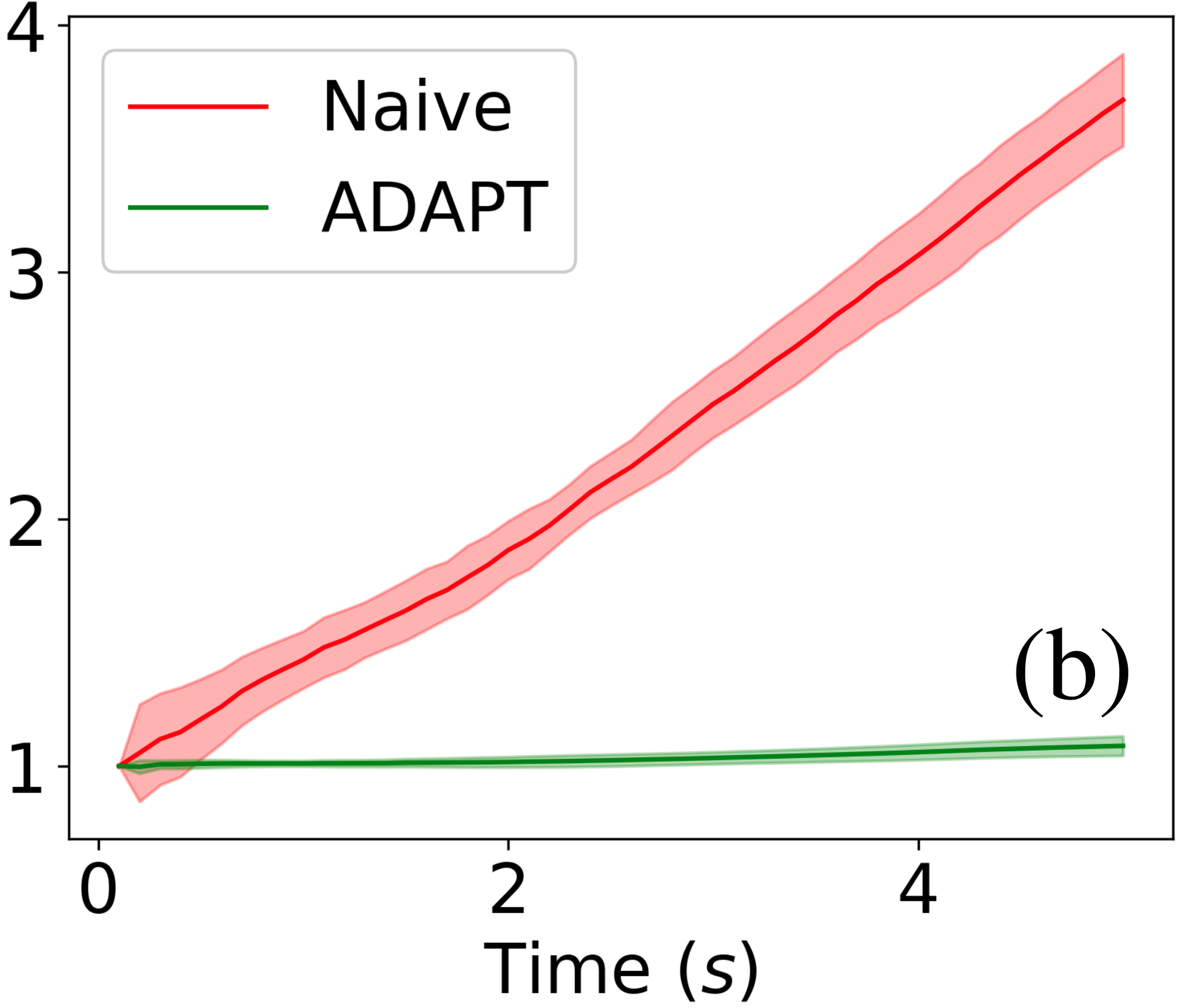}
    \end{subfigure}%
    \begin{subfigure}[t]{0.33\textwidth}
        \centering
        \includegraphics[scale=0.054]{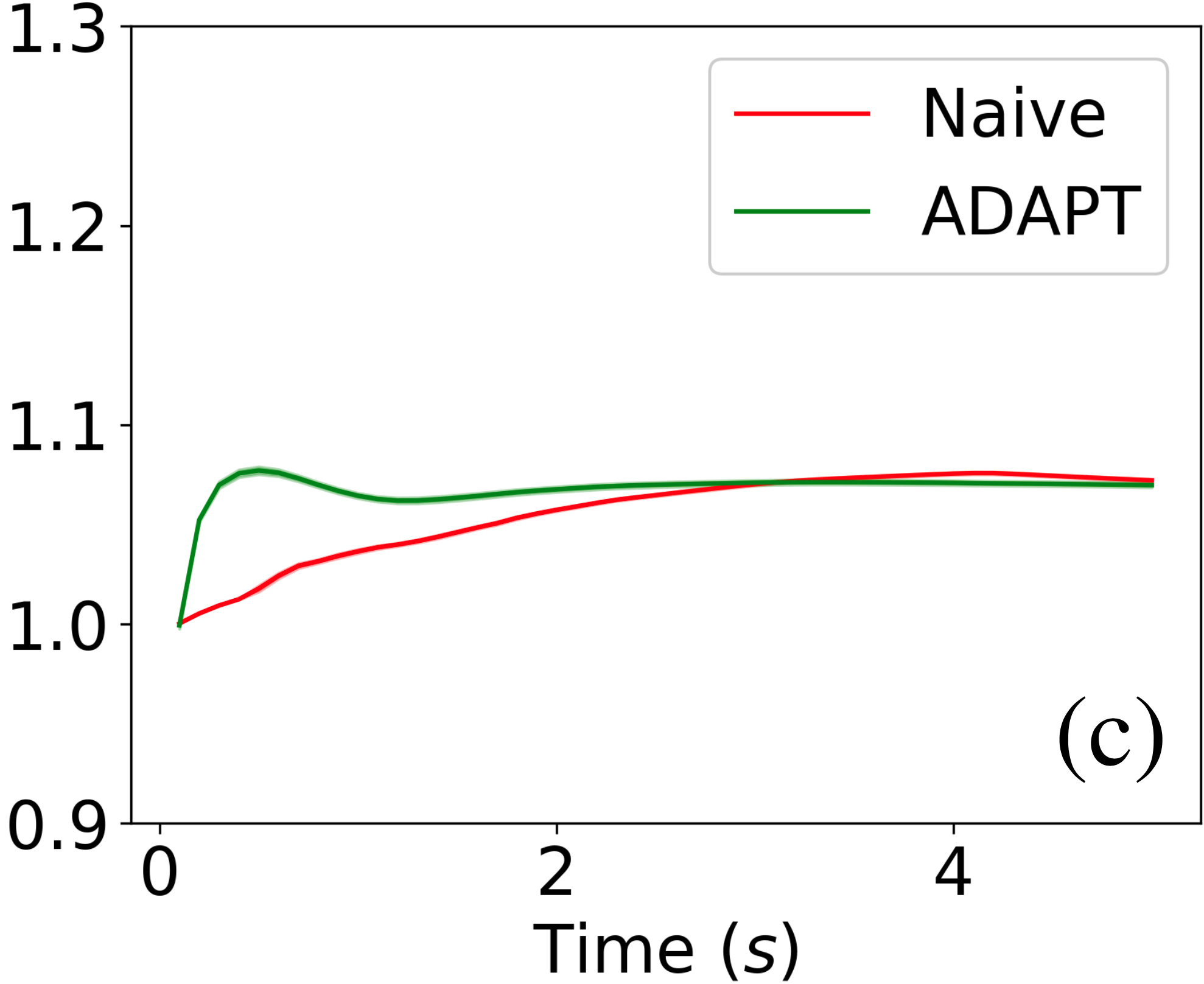}
        \label{fig:reacher_param}
    \end{subfigure}%
    \vspace{-5pt}
    \caption{Mean cumulative cost over the length of an episode for 50 episodes on the reacher environment. The confidence intervals are standard error. The costs are normalized to the cost of the naive policy being rolled out on the simulated environment from the same initial state, to allow more direct comparison across episodes. The \textit{naive} rollout is the nominal policy executed on the target environment. The disturbances tested are \textbf{a)} additive control error, \textbf{b)} process noise, and \textbf{c)} dynamics parameter error.}
    \label{fig:reacher_cost}
\end{figure}


\section{Conclusion and Outlook} 
\label{sec:conclusion}

We have presented the \algoName algorithm for robust transfer of learned policies to target environments with unmodeled disturbances or model parameters. We have also provided guarantees on the lower bounds of the accrued reward in the target environment for a policy transferred with \algoName. Our results were demonstrated on two different environments with four disturbance models investigated. We additionally discuss usage of robust policies with \algoName. The results presented demonstrate that this method improves performance on unmodeled disturbances by 50-300\%.

In this work, we construct our analysis on the Lipschitz continuity of the dynamics. Indeed, the smoothness of the deviation in dynamics is fundamental to the guarantees we establish. An immediate avenue of future investigation is, therefore, expanding the work presented here to environments with discrete and discontinuous dynamics such as contact. Recently, \textcite{farshidian2016sequential} have extended an iteratively linearized nonlinear MPC, similar to ours, to switching linear systems, which may have potential as a foundation on which to develop a capable contact formulation of \algoName. Additionally, recent work has developed robust, receding horizon tube controllers that allow the establishment of explicit tubes in the state space \cite{singh2017robust}. This approach has the potential to establish explicit safety constraints for operation in cluttered environments. Finally, these methods will also be evaluated on a physical systems. 



\renewcommand{\baselinestretch}{0.98} 
\renewcommand*{\bibfont}{\small}
\printbibliography 

\end{document}